\pdfoutput=1 

\documentclass[11pt]{article}

\usepackage[utf8]{inputenc} % allow utf-8 input
\usepackage[T1]{fontenc}    % use 8-bit T1 fonts
\usepackage{url}            % simple URL typesetting
\usepackage{booktabs}       % professional-quality tables
\usepackage{amsfonts}       % blackboard math symbols
\usepackage{nicefrac}       % compact symbols for 1/2, etc.
\usepackage{microtype}      % microtypography
\usepackage{amsmath}        % amsmath
\usepackage{amsthm}         % amsthm
\usepackage{mathtools}      % mathtools
\usepackage{float}
\usepackage{enumitem}
\usepackage{xspace}
\usepackage{fullpage}
\usepackage[dvipsnames]{xcolor}
\usepackage{hyperref}
\usepackage{pdfsync}

\usepackage{libertine}
\usepackage{libertinust1math}
\usepackage{inconsolata}

\usepackage[noend,ruled,linesnumbered,algo2e]{algorithm2e}

\newcommand{\todo}[1]{}
\renewcommand{\chi}[1]{{\bf \color{blue} CHI: #1}}

\newcommand{\declarecolor}[2]{\definecolor{#1}{RGB}{#2}\expandafter\newcommand\csname #1\endcsname[1]{\textcolor{#1}{##1}}}
\declarecolor{White}{255, 255, 255}
\declarecolor{Black}{0, 0, 0}
\declarecolor{LightGray}{216, 216, 216}
\declarecolor{Gray}{127, 127, 127}
\declarecolor{Orange}{237, 125, 49}
\declarecolor{LightOrange}{251,229, 214}
\declarecolor{Yellow}{255, 192, 0}
\declarecolor{LightYellow}{255, 242, 200}
\declarecolor{Blue}{91, 155, 213}
\declarecolor{LightBlue}{222, 235, 247}
\declarecolor{Green}{112, 173, 71}
\declarecolor{LightGreen}{226, 240, 217}
\declarecolor{Navy}{68, 114, 196}
\declarecolor{LightNavy}{218, 227, 243}
\declarecolor{DarkBlue}{0,20,115}

\hypersetup{
	colorlinks=true,
	pdfpagemode=UseNone,
	citecolor=Navy,
	linkcolor=Navy,
	urlcolor=Navy,
}

\def\expec#1#2{{\mathbb{E}}_{#1}\left[ #2 \right]}

\newcommand{\SC}{\textsc{SchurComplement}\xspace}
\newcommand{\RC}{\textsc{RandomContraction}\xspace}
\newcommand{\diag}{\textrm{diag}\xspace}

\newcommand{\blogcatalog}{BlogCatalog\xspace}
\newcommand{\flickr}{Flickr\xspace}
\newcommand{\youtube}{YouTube\xspace}
\newcommand{\NetMF}{NetMF\xspace}
\newcommand{\LINE}{LINE\xspace}
\newcommand{\deepwalk}{DeepWalk\xspace}

\newcommand{\mat}[1]{\mathbf{#1}}
\newcommand{\trunclog}{\textrm{log}^{+}}
\newcommand{\SCf}{\mathbf{SC}\xspace}

\theoremstyle{plain}
\newtheorem{theorem}{Theorem}[section]%[chapter]
\newtheorem{fact}[theorem]{Fact}
\newtheorem{lemma}[theorem]{Lemma}

\newtheorem{definition}[theorem]{Definition}

\usepackage{parskip}

\usepackage[round]{natbib}
\begin{document}

\title{Faster Graph Embeddings via Coarsening}
\author{
  Matthew Fahrbach%
  \footnote{Equal contribution}
  \thanks{Google Research.
  Email: \href{mailto:fahrbach@google.com}{fahrbach@google.com}}
  \hspace{-0.36cm}
  \and
  Gramoz Goranci%
  \footnotemark[1]
  \thanks{University of Toronto.
  Email: \href{mailto:ggoranci@cs.toronto.edu}{ggoranci@cs.toronto.edu}}
  \hspace{-0.36cm}
  \and
  Richard Peng%
  \footnotemark[1]
  \thanks{Georgia Institute of Technology.
  Email: \href{mailto:rpeng@cc.gatech.edu}{rpeng@cc.gatech.edu}}
  \hspace{-0.36cm}
  \and
  Sushant Sachdeva%
  \footnotemark[1]
  \thanks{University of Toronto.
  Email: \href{mailto:sachdeva@cs.toronto.edu}{sachdeva@cs.toronto.edu}}
  \hspace{-0.36cm}
  \and
  Chi Wang%
  \footnotemark[1]
  \thanks{Microsoft Research.
  Email: \href{mailto:wang.chi@microsoft.com}{wang.chi@microsoft.com}}
}
\date{\today}

\maketitle

\begin{abstract}
Graph embeddings are a ubiquitous tool for machine learning tasks, such as node classification and link prediction, on graph-structured data. However, computing the embeddings for large-scale graphs is prohibitively inefficient even if we are interested only in a small subset of relevant vertices. To address this, we present an efficient graph coarsening approach, based on Schur complements, for computing the embedding of the relevant vertices. We prove that these embeddings are preserved exactly by the Schur complement graph that is obtained via Gaussian elimination on the non-relevant vertices. As computing Schur complements is expensive, we give a nearly-linear time algorithm that generates a coarsened graph on the relevant vertices that provably matches the Schur complement in expectation in each iteration. Our experiments involving prediction tasks on graphs demonstrate that computing embeddings on the coarsened graph, rather than the entire graph, leads to significant time savings without sacrificing accuracy.
\end{abstract}

%%% Local Variables:
%%% mode: latex
%%% TeX-master: "vertex-sparsification"
%%% End:

\section{Introduction}

Over the past several years, network embeddings have been demonstrated
to be a remarkably powerful tool for learning unsupervised
representations for nodes in a network~\citep{PerozziAS14,
  tang2015line, GroverL2016}.
Broadly speaking, the objective is to learn a low-dimensional vector
for each node that captures the structure of its neighborhood.
These embeddings have proved to be very effective for downstream machine
learning tasks in networks such as node classification and link
prediction~\citep{tang2015line, HamiltonYL2017}.

While some of these graph embedding approaches are explicitly based on
matrix-factorization~\citep{Tang2011, BrunaZSL14, Cao2015grarep},
some of the other popular methods, such as DeepWalk~\citep{PerozziAS14} and
LINE~\citep{tang2015line}, can be viewed as approximately factoring
random walk matrices constructed from the graph.
A new approach proposed by \citet{QiuDMLWT18} called NetMF explicitly
computes a low-rank approximation of random-walk matrices using a
Singular Value Decomposition (SVD),
% and demonstrate that it
%
% The authors demonstrate that NetMF
and significantly outperforms the DeepWalk and LINE embeddings for
benchmark network-mining tasks.

Despite the performance gains, explicit matrix factorization results
in poor scaling performance.
The matrix factorization-based approaches typically require computing
the singular value decomposition (SVD) of an $n \times n$ matrix,
where $n$ is the number of vertices in the graph.
In cases where this matrix is constructed by taking several steps of
random walks, e.g., NetMF~\citep{QiuDMLWT18}, the matrix is often dense even
though the original graph is sparse.
This makes matrix factorization-based embeddings extremely expensive
to compute, both in terms of the time and memory required,
even for graphs with 100,000 vertices.
% Since the factorization step requires like $\Omega(n^3),$
% and practically, at least $\Omega(n^2),$ making it extremely expensive
%

% The sizes of graphs datasets that we would like to study has grown
% very rapidly in the past few years.

There are two main approaches for reducing the size of a graph to
improve the efficiency and scalability of graph-based learning.
The first method reduces the number of edges in the graphs while preserving
properties essential to the relevant applications. This approach is
often known as graph sparsification (see~\citet{BatsonSST13} for a
survey).
Recently,~\citet{QiuDMLWWT19} introduced NetSMF, a novel approach for
computing embeddings that leverages spectral sparsification for random
walk matrices~\citep{ChengCLPT15-colt} to dramatically improve the
sparsity of the matrix that is factorized by NetMF,
resulting in improved space and time efficiency, with comparable
prediction accuracy.

The second approach is \emph{vertex sparsification}, i.e., eliminating
vertices, also known as graph coarsening.
% \citet{LiangGP18} designed a hierarchical network embedding
% method based on vertex reduction.
% Vertex reduction has been studied in the context of
% multiscale or multilevel
% algorithms~\cite{ScnneiderFPKW00,RonSB11,LiangGP18} for Network
% embeddings.
However, there has been significantly less rigorous treatment of
vertex sparsification for network embeddings,
%
% The task of eliminating vertices (i.e., vertex sparsification)
% has received comparatively less attention.
% There has been
% significantly less rigorous treatment for vertex sparsification compared to
% edge sparsification.
% Vertex sparsification for network embeddings 
% We are interested in settings where
which is useful for many downstream tasks that only require
embedding a relevant subset of the vertices, e.g., (1) core-peripheral
networks where data collection and the analysis focus on a subset of
vertices~\citep{BorgattiE00}, (2) clustering or training models on a
small subset of
% labels using a subset of
representative vertices~\citep{Karypis98}, and (3) directly working
with compressed versions of the graphs~\citep{LiuDSK16}.

In all of the above approaches to graph embedding, the only way to obtain
an embedding for a subset of desired vertices is to
first compute a potentially expensive embedding for the entire graph
and then discard the embeddings of the other vertices.
Thus, in situations where we want to
perform a learning task on a small fraction of nodes in the graph,
this suggests the approach of computing a graph embedding
on a smaller proxy graph on the target nodes
that maintains much of the connectivity structure of the original network.

% Even if the tasks requires the embeddings of only a small subset of
% the vertices, the current embedding approaches require computing the
% embedding for all the vertices in the graph.
%
% In all of these
% settings, it is important to remove vertices while retaining as much
% information as possible about the underlying graph.

\subsection{Our Contributions}
In this paper, we present efficient vertex sparsification algorithms for preprocessing
massive graphs in order to reduce their size while preserving network
embeddings for a given relevant subset of vertices.
Our main algorithm repeatedly chooses a non-relevant vertex to remove and
contracts the chosen vertex with a random neighbor, while reweighting
edges in its neighborhood.
% add edges between its neighbors.
This algorithm provably runs in nearly linear time in the size of the
graph,
and we prove that in each iteration, in expectation the algorithm performs Gaussian
elimination on the removed vertices, adding a weighted clique on the
neighborhood of each removed vertex, computing what is known as the
\emph{Schur complement} on the remaining vertices.
Moreover, we prove that the Schur complement is guaranteed to exactly preserve
the matrix factorization that random walk-based graph embeddings seek
to compute as the length of the random walks approaches infinity.

When eliminating vertices of a graph using the Schur complement, the
resulting graph perfectly preserves random walk transition
probabilities through the eliminated vertex set with respect to the
original graph.
Therefore, graph embeddings that are constructed by taking small length random
walks on this sparsified graph are effectively taking longer random walks
on the original graph, and hence can achieve comparable or improved prediction
accuracy in subsequent classification tasks
while also being less expensive to compute.
 
%
% exactly for NetMF~\cite{QiuDMLWT18} or
% approximately~\cite{PerozziAS14, tang2015line}.
% %

% {\color{red} Written the above part so far.}

% \sushant{Rewrite the below line to include time speedup and numbers}

Empirically, we demonstrate several advantages of our algorithm on
widely-used benchmarks for the multi-label vertex classification and
link prediction.
We compare our algorithms using LINE~\citep{tang2015line}, NetMF~\citep{QiuDMLWT18},
and NetSMF~\citep{QiuDMLWWT19} embeddings.
%
%
% Our algorithms lead to a significant reduction in the size of the
% resulting graph, resulting in significantly lower memory and time 
%
Our algorithms lead to significant time improvements, especially on
large graphs (e.g., 5x speedup on computing NetMF on the YouTube
dataset that has ~1 million vertices and ~3 million edges).
In particular, our
randomized contraction algorithm is extremely efficient and runs
in a small fraction of the time required to compute the embeddings.
% runs
% in time that is under 1-2\% of the time taken to fastest of these
% embedding methods on the original graphs.
%
By computing network embeddings on the reduced graphs instead of the
original networks, our algorithms also result in at least comparable, if
not better, accuracy for the multi-label vertex classification
and AUC scores for link prediction.
%

%
%  for multi-label
% vertex classification 
% Computing the embeddings on the reduced graphs instead of the original
% networks (with the same random walk length),

% Compared to original
% network, This results in improved prediction accuracy for multi-label
% vertex classification by computing network embeddings on the reduced
% graphs instead of the original networks (with the same random walk
% length).

% on the original graph.

% Additionally, we demonstrate on
% widely-used benchmarks that both of these algorithms empirically improve the
% prediction accuracy for multi-label vertex classification compared to using
% graph embeddings of the original and unsparsified networks.

\subsection{Other Related Works}
\label{subsec:Related}

% Provable spectral approximations to the Schur complement have been
% studied in the context of efficiently solving linear
% systems~\cite{KyngS16, KyngLPSS16}. Compared to the original graph,
% the approximation constructed by these algorithms has many more edges
% (by at least a factor of $1/\varepsilon^2$),
% % Moreover, \cite{KyngLPSS16} uses weighted expanders at each step to
% % establish their theoretical guarantees,
% limiting the practical applicability of these works.
% %
% On the other hand, our work introduces Schur complements in the
% context of graph embeddings, and we show that a simple random
% contraction rule leads to a decrease in the edge count in the
% contracted graph,
% % preserves Schur complements in expectation in each
% % step,
% and performs well in practice.

The study of vertex sparsifiers is closely related to graph
coarsening~\citep{chevalier2009comparison, LoukasV18} and the study of
core-peripheral networks~\citep{BensonK18:arxiv,JiaB19}, where
analytics are focused only on a core of vertices.  In our setting, the
\emph{terminal vertices} play roles analogous to the core vertices.

In this paper, we focus on unsupervised approaches for learning graph
embeddings, which are then used as input for downstream classification
tasks.
There has been considerable work on semi-supervised approaches to
learning on graphs~\citep{YangCS16, kipf2016semi, velickovic18}, including some that exploit connections with
Schur complements~\citep{vattani2011preserving, WagnerGKM18,
  viswanathan2019}.

Our techniques have direct connections with multilevel and multiscale
algorithms, which aim to use a smaller version of a problem (typically
on matrices or graphs) to generate answers that can be extended to the
full problem~\citep{ChenPHS18, LiangGP18, pmlr-v97-abu-el-haija19a}.
There exist well-known connection between Schur complements,
random contractions, and finer grids in the multigrid
literature~\citep{BriggsHM00:book}.
These connections have been utilized for efficiently solving Laplacian
linear systems~\citep{KyngS16, KyngLPSS16}, via provable spectral
approximations to the Schur complement.
However, approximations constructed using these algorithms have many
more edges (by at least a factor of $1/\varepsilon^2$) than the
original graph, limiting the practical applicability of these works.
On the other hand, our work introduces Schur complements in the
context of graph embeddings, and gives a simple random contraction
rule that leads to a decrease in the edge count in the contracted
graph, preserves Schur complements in expectation in each step, and
performs well in practice.

% However, our study focuses more on the structural properties
% of these graphs with fewer vertices instead of their integration
% into multiscale graph embedding
% algorithms.

Graph compression techniques aimed at reducing the number of vertices have been studied for other graph primitives, including cuts/flows~\citep{Moitra09,EnglertGKRTT14} and shortest path distances~\citep{ThorupZ05}. However, the main objective of these works is to construct sparsifiers with theoretical guarantees and to the best of our knowledge, there are no works that consider their practical applicability.

\section{Preliminaries}
\label{sec:Preliminaries}
We introduce the notion of graph embeddings and graph coarsening.
In the graph embedding problem, given an undirected, weighted graph $G=(V,E,w)$, where $V$ is the vertex set of $n$ vertices, and $E$ is the edge set of $m$ edges, the goal is to learn a function $f : V \rightarrow \mathbb{R}^{d}$ that maps each vertex to a $d$-dimensional vector while capturing structural properties of the graph. An important feature of graph embeddings is that they are independent of the vertex labels, i.e., they are learned in an unsupervised manner. This allows us to perform supervised learning by using the learned vector representation for each vertex, e.g., classifying vertices via logistic regression.

In this work, we study the matrix factorization based approach for graph embeddings introduced by~\citet{QiuDMLWT18}. Assume that vertices are labeled from $1$ to $n$. Let $\mat{A}$ be the adjacency matrix of $G$, and let $\mat{D} = \diag(d_1,\ldots,d_n)$ be the degree matrix, where $d_i = \sum_{j} \mat{A}_{ij}$ is the \emph{weighted} degree of the $i$-th vertex. A key idea of random walk-based graph embeddings is to augment the matrix that will be factored with longer random walks. A unified view of this technique, known as  Network Matrix Factorization~(\text{NetMF}), is given below
\begin{equation}
\textrm{SVD} \left(\trunclog \left( \sum_{i=1}^{W} \theta^{i} \mat{D}^{-1} \left( \mat{A} \mat{D}^{-1} \right)^{i} \right), d \right),
\end{equation}
where $d$ is the target dimension, $W$ is the \emph{window size},
$\theta_1,\ldots,\theta_W \in (0,1)$ are fixed parameters, and $\trunclog$ is the \emph{entry-wise} truncated logarithm defined as $\trunclog(x) := \max\left( \log(m \cdot x),0 \right)$. 

\citet{QiuDMLWT18} showed that NetMF is closely related to the DeepWalk model, introduced by~\citet{PerozziAS14} in their seminal work on graph embeddings. \text{NetMF} also generalizes the
LINE graph embedding algorithm~\citep{tang2015line}, which is equivalent to \text{NetMF} with $W=1$. 

In graph coarsening (also known as vertex sparsification), given an undirected, weighted graph $G=(V,E,w)$ and a subset of relevant vertices, referred to as \emph{terminals}, $T \subseteq V$, the goal is construct a graph $H$ with fewer vertices that contains the terminals while preserving important features or properties of $G$ with respect to the terminals $T$. %\gramoz{Probably include some examples, and say that in this work with deal with coarsening based on gaussian elimination?}

An important class of matrices critical to our coarsening algorithms are SDDM matrices. A Matrix $\mat{M}$ is a \emph{symmetric diagonally dominant \emph{M}-matrix} (SDDM) if $\mat{M}$ is
(i) symmetric,
(ii) every off-diagonal entry is non-positive, and
(iii) diagonally dominant, i.e., for all $i \in [n]$ we have
$\mat{M}_{ii} \ge -\sum_{j \ne i} \mat{M}_{ij}$.
%$-\sum_{j \neq i} \mat{M}_{ij} \leq \mat{M}_{ii}$.
An SDDM matrix $\mat{M}$ can also be written as a Laplacian matrix $\mat{L} := \mat{D} - \mat{A}$ plus a non-negative, non-zero, diagonal matrix $\mat{D}^{s}$.

%%% Local Variables:
%%% mode: latex
%%% TeX-master: "vertex-sparsification"
%%% End:

\section{Graph Coarsening Algorithms}
\label{sec:Algorithms}

In this section, we present the graph coarsening algorithms.
Our first algorithm is based on Gaussian elimination,
where we start with an SDDM matrix and form its Schur complement
via row and column operations.
Next, we design an algorithm for undirected,
weighted graphs with loops,
which translates these matrix operations into graph operations.

Let $\mat{M}$ be an SDDM matrix and recall that by definition $\mat{M} = \mat{L} + \mat{D}^{s}$, where $\mat{L} := \mat{D} - \mat{A}$ is the Laplacian matrix associated with some undirected, weighted graph $G=(V,E,w)$ and $\mat{D}^{s}$ is the slack diagonal matrix, which corresponds to self-loops in $G$. Let $\mat{D}' = \mat{D} + \mat{D}^s$. For an edge $e=(u,v) \in E$, let $\mat{L}(u,v) = (\mat{1}_u - \mat{1}_v)(\mat{1}_u - \mat{1}_v)^{\top}$ denote the Laplacian of $e$, where $\mat{1}_u$, $\mat{1}_v$ are indicator vectors. The Laplacian matrix is also given by $\mat{L} := \sum_{e \in E} w(e) \mat{L}(e)$. The \emph{unweighted} degree of a vertex $x$ in $V$ is the number of edges incident to $x$ in $G$.

We now consider performing one step of Gaussian elimination. Given a matrix $\mat{M}$ and a vertex $x \in V(G)$ that we want to eliminate, assume without loss of generality that the first column and row of $\mat{M}$ correspond to $x$. The \emph{Schur complement} of~$\mat{M}$ with respect to $T := V \setminus \{x\}$ is given by
\begin{equation}
\label{eq:SC_SDD}
\SCf \left(\mat{M}, T \right) = \mat{M}_{T, T} - \frac{\mat{M}_{T,x} \mat{M}_{T,x}^{\top}}{\mat{D}^{\prime}_{x,x}}.
\end{equation}   

An important feature of the output matrix is that it is an SDDM
matrix~(see supplementary material), i.e., it correspond to a graph on
$T$ with self loops. This suggests that there should exist a reduction
rule allowing us to go from the original graph to the reduced
graph. We next present a way to come up with such a rule by re-writing
the Schur complement in Eq.~\eqref{eq:SC_SDD}, which in turn leads to
our first graph coarsening routine \SC given in
Algorithm~\ref{alg:schur}. Note that this routine iteratively
eliminates every vertex in $V \setminus T$ using the same reduction
rule.

\begin{algorithm2e}
	\caption{\SC}
  \label{alg:schur}
	
	\BlankLine
	\KwData{graph $G=(V,E,w)$ given as $\mat{D} - \mat{A} + \mat{D}^{s}$, terminals $T \subseteq V$, degree threshold $\Delta$}
	\KwResult{vertex sparsifier $H$ of $G$ such that $T \subseteq V_H$}
	\BlankLine

	Set $H \gets G$ \\
  \While{\textnormal{there exists a vertex $x \in V_H \setminus T$ with unweighted degree
  $\leq \Delta$}}
	{
		Let $x$ be the minimum degree vertex in $V_H \setminus T$\\
    \For{\textnormal{each vertex} $u \in N(x)$}
		{
      \For{\textnormal{each vertex} $v \in N(x)$} 
		{
			Add edge $(u,v)$ to $H$ with weight
			$\left( w(x, u) w(x, v) \right)/\mat{D}'_{x,x}$
      \label{line:clique}
		}
		
		Set $\mat{D}^s_{u,u} \gets \mat{D}^s_{u,u} + \left(w(x,u) \cdot \mat{D}^s_{x,x} \right)/\mat{D}'_{x,x}$ \label{line:slack}
		
		%Increase $D_s(u,u)$ by $w(x,u) - \sum_{y \in N(x)} w(x,y) w(x,u)/\left(\sum_{y \in N(x)} w(x,y) + D_s(x,x) \right)$
		} 
	  Remove vertex $x$ from $H$ \label{line:removal}
	}
	\Return $H$
\end{algorithm2e}

Given a Laplacian $\mat{L}$, let $\mat{L}^{(v)}$ denote the Laplacian corresponding to the edges incident to vertex $v$, i.e., $\mat{L}^{(v)} = \sum_{e \in E : e \ni v} w(e)\mat{L}(e)$. If the first column of $\mat{L}$ can be written, for some vector $\mat{a}$, as
\[ \begin{bmatrix} \mat{D}'_{x,x} \\ -\mat{a} \end{bmatrix}, \text{~~~~~then~~~~~} \mat{L}^{(x)} = \begin{bmatrix} \mat{D'}_{x,x} & -\mat{a}^{\top} \\ -\mat{a} & \diag(\mat{a}) \end{bmatrix}. \]

Using these definitions, observe that the
first term in Eq.~\eqref{eq:SC_SDD} can be re-written as follows
\begin{equation}
\label{eq:removal}
  \mat{M}_{T,T} = \left(\mat{L} - \mat{L}^{(x)}\right)_{T,T} + \diag(\mat{a}) + \mat{D}^{s}_{T,T}.
\end{equation}

The first two terms in Eq.~\eqref{eq:removal} give that the vertex $x$ must be deleted from the underlying graph $G$~(Line~\ref{line:removal} in Algorithm~\ref{alg:schur}). Next, the second term in Eq.~\eqref{eq:SC_SDD} together with (i) $\mat{a} = \mat{M}_{T,x}$ and (ii) $(\diag(\mat{a}) \cdot \mat{D}_{x,x})/\mat{D}'_{x,x}$ give us
\begin{align}
\label{eq:clique}
	\frac{ \diag(\mat{a})  \cdot \mat{D}_{x,x}}{\mat{D}'_{x,x}} - \frac{\mat{a} \mat{a}^{\top}}{\mat{D}'_{x,x}} 
	= \frac{1}{2}\sum_{u \in N(x)}\sum_{v \in N(x) \setminus \{u\}} \frac{w(x,u)w(x,v)}{\mat{D}'_{x,x}} \mat{L}(u,v),
\end{align}
which corresponds to the weighted \emph{clique} structure formed by Schur complement~(Line~\ref{line:clique} in Algorithm~\ref{alg:schur}). Finally, the remaining terms in
Eq.~\eqref{eq:removal} together with the rescaled matrix
$-(\diag(\mat{a}) \cdot \mat{D}_{x,x})/\mat{D}'_{x,x}$ give
\begin{align*}
	\mat{D}^{s}_{T,T} + \diag(\mat{a}) - \frac{\diag(\mat{a})\mat{D}_{x,x}}{\mat{D}'_{x,x}}
	 = \mat{D}^{s}_{T,T} + \diag(\mat{a})\frac{\mat{D}^{s}_{x,x}}{\mat{D}'_{x,x}},
\end{align*}
which corresponds to the rule for updating loops for the neighbors of $x$~(Line~\ref{line:slack} in Algorithm~\ref{alg:schur}). This completes the reduction rule for eliminating a single vertex.

%along with the definition of
%symmetric diagonally dominant M-matrices (SDDM matrices)
%(e.g. Horn and Johnson~\cite{HornJ12}, Section 8).

%\begin{fact}
%	\label{fact:SchurInverse}
%	For any positive-definite matrix $M$, and any subset of variables $T$,
%	we have
%	\[
%	\SC\left(M, T \right)^{-1}
%	= \left(M^{-1} \right)_{T, T}
%	\]
%\end{fact}

%\gramoz{This algorithm is adding $n^2$ edges to the graph? Shouldn't we divide they weight by $1/2$? Are we assuming that $w(u,u) =0$?}

While \SC is highly efficient when the degree is small, its cost can
potentially become quadratic in the number of vertices.
In our experiments, we delay this explosion in edge count as much as possible
by performing the widely-used minimum degree heuristic: repeatedly eliminate the
vertex of the smallest degree~\citep{george1989evolution,fahrbach2018graph}.
However, because the increase in edges is proportional to the degree squared,
on many real-world graphs this heuristic exhibits a phenomenon similar to a
phase transition---it works well up to a certain point and then it is suddenly unable to make further progress.  To remedy this, we study the opposite extreme:
a contraction-based scheme that does not create any additional edges.

Given a graph $G$ and terminals $T \subseteq V(G)$, the basic idea behind
our second algorithm is to repeatedly pick a minimum degree vertex $x \in V \setminus T$, sample a neighbor $u \in N(x)$ with probability proportional to the edge weight $w(x,u)$, contract $(x,u)$ and then reweight the new edges incident to the chosen neighbor. We formalize this notion in the following definition.
\vspace{0.25cm}
\begin{definition}[Random Contraction] \label{def:contraction}
 Let $G=(V,E,w)$ be a graph with terminals $T \subseteq V$.
  Let $x \in V \setminus T$ be a non-terminal vertex. Let $H_x$ be the random star generated by the following rules:
\begin{enumerate}[noitemsep, nolistsep]
\item Sample a neighbor $u \in N(x)$ with probability $w(x,u)/ \mat{D}_{x,x}$.
\item Contract the edge $(x,u)$.
\item For each edge $(u,v)$, where $v \in N(x) \setminus \{u\}$ in $H_x$, set $w(u,v)$ to be $\frac{w(x,u) w(x,v)}{w(x,u) + w(x,v)} \cdot \left(\mat{D}_{x,x} / \mat{D}'_{x,x} \right)$.
\end{enumerate}
Let $H$ be the sparsified graph obtained by including $H_x$ and removing $x$, i.e.,
$H:= (G \setminus \{x\}) \cup H_x$.
\end{definition}

As we will shortly see, the time for implementing such reweighted random
contraction for vertex $x$ is linear in the degree of~$x$. This is much
faster compared to the Schur complement reduction rule that requires time
quadratic in the degree of $x$. Another important feature of our randomized
reduction rule is that it preserves the Schur complement in expectation in each iteration.
Repeatedly applying such a rule for all non-terminal vertices leads to the
procedure presented in Algorithm~\ref{alg:contract}.

\begin{algorithm2e}
	\caption{\textsc{RandomContraction}}
	\label{alg:contract}	
	\BlankLine
	\KwData{graph $G=(V,E, w)$ given as $\mat{D} - \mat{A} + \mat{D}^s$, terminals $T \subseteq V$, degree threshold $\Delta$}
	\KwResult{sparsifier $H$ of $G$ that approximates $\SC(G,T)$}
	\BlankLine
	Set $H \gets G$ \\
  \While{\textnormal{there exists a vertex $x \in V_H \setminus T$ with unweighted degree $\leq \Delta$}}
	{		
		Let $x$ be the minimum degree vertex in $V_H \setminus T$ \\
        \For{\textnormal{each vertex $u \in N(x)$}}
	    {
	    	 \label{line:slackUpdate} Set $\mat{D}^s_{u,u} \gets \mat{D}^s_{u,u} + \left(w(x,u) \cdot \mat{D}^{s}_{x,x}\right)/\mat{D}'_{x,x}$
	    	
	    }		
	    Contract the edge $(x,u)$, where $u \in N(x)$, with probability $w(x,u)/\mat{D}_{x,x}$ \\
      \For{\textnormal{each edge $(u,v)$, where $v \in N(x) \setminus \{u\}$}}
	    {
        Set $w(u,v) \gets w(u,v) + \frac{w(x,u) w(x,v)}{w(x,u) + w(x,v)} \cdot \left( \frac{\mat{D}_{x,x}}{\mat{D}'_{x,x}} \right) $
	    }		
		Let $H'$ be the resulting graph and set $H \gets H'$ \\
	}
	\Return $H$
\end{algorithm2e}

Now we analyze the behavior of our contraction-based algorithm \RC.
The following theorem demonstrates why it can be significantly more efficient
than computing Schur complements, while still preserving the Schur complement
in expectation in each iteration. In what follows, whenever we talk about a graph $G$, we assume that is given together with its associated SDDM matrix~$\mat{M}$.

\vspace{0.25cm}
\begin{theorem} \label{lem:randomContract}
	Given a graph $G=(V,E,w)$ with $m$ edges and terminals $T \subseteq V$, the algorithm \RC produces a graph $H$ with $O(m)$ edges that contains the terminals $T$ in $O(m \log n)$ time. Moreover, $H$ preserves $\SCf(\mat{M},T)$ in expectation in each iteration.
\end{theorem}

Before proving Lemma~\ref{lem:randomContract}, we first analyze the scenario
of removing a single non-terminal vertex using a reweighted random contraction.

\vspace{0.25cm}
\begin{lemma}
\label{lem:singleContract}
  Let $G=(V,E,w)$ be a graph with terminals $T \subseteq V$. Let $x \in V
  \setminus T$ be a non-terminal vertex. Let $H$ be the
  sparsifier of $G$ from Definition~\ref{def:contraction} and assume that the slacks of the neighbors of $x$ are updated according to the rule in Line~\ref{line:slackUpdate} of Algorithm~\ref{alg:contract}.
  Then we have
	\[ \expec{}{\mat{M}_H}
	=
	\SCf\left(\mat{M}, V \setminus \left\{x\right\}\right).
	\] 
  Furthermore, $H$ can be computed in $O(\deg(x))$ time.
\end{lemma}

\begin{proof}
	We first show that $H$ preserves the Schur complement in expectation.
	By Eq.~\eqref{eq:SC_SDD} and the follow up discussion in Section~\ref{sec:Algorithms}, we know that taking the Schur complement with respect to $V \setminus \{x\}$ corresponds to (i) deleting $x$ together with its neighbors from $G$, (ii) updating the slacks of the neighbors of $x$ and (iii) introducing a clique among neighbors of $x$ and adding it to $G$. Note that $x$ is contracted to one of its neighbors in $H$, i.e., it is deleted from $G$, and the rule for updating the slacks in both \SC and \textsc{RandomContraction} is exactly the same. Thus it remains to show that the random edge contraction in $H$ preserves the clique structure of Schur complement in expectation. 
	
	To this end, recall from Eq.~\eqref{eq:clique} that the clique structure of Schur complement is given by
	\[
		\frac{1}{2}\sum_{u \in N(x)}\sum_{v \in N(x) \setminus \{u\}} \frac{w(x,u)w(x,v)}{\mat{D}'_{x,x}} \mat{L}(u,v).
	\]  	   
	For each $u \in N(x)$, let $H^{x \rightarrow u}$ be the weighted star that would be formed if $x$ gets contracted to $u$, that is
	\[
	\mat{M}^{x \rightarrow u}_{H}
	:= \sum_{v \in N(x) \setminus \{u\}} \frac{w(x,u) w(x,v)}{w(x,u) + w(x,v)} \left( \frac{\mat{D}_{x,x}}{\mat{D}'_{x,x}} \right) \mat{L}(u,v).
	\]
	%By Definition~\ref{def:contraction}, it follows that $\expec{}{L(H)} = L(G \setminus {x}) + \expec{u}{L(H(x \rightarrow u))}$,
	The probability that the edge $(x,u)$ is contracted is $w(x,u)/\mat{D}_{x,x}$ by Definition~\ref{def:contraction}.
	As a result, we obtain the following equality
	\begin{align*}
	 \expec{u}{\mat{M}_H^{x \rightarrow u}}
     &= \sum_{u \in N(x)} \frac{w(x,u)}{\mat{D}_{x,x}} \mat{M}^{x \rightarrow u}_{H} \\
     &= \sum_{u \in N(x)} \frac{w(x,u)}{\mat{D}_{x,x}}
     \cdot \sum_{v \in N(x) \setminus \{u\}} \frac{w(x,u) w(x,v)}{w(x,u) + w(x,v)} \left( \frac{\mat{D}_{x,x}}{\mat{D}'_{x,x}} \right) \mat{L}(u,v) \\
    &= \frac{1}{2}\sum_{u \in N(x)}\sum_{v \in N(x) \setminus \{u\}} \frac{w(x,u)w(x,v)}{\mat{D}'_{x,x}} \mat{L}(u,v).
	\end{align*}
	
  For bounding the running time for computing $H$, reweighting the star $H(x)$ takes $O(\deg (x))$ time. We can also simulate the random edge incident to
  $x$ by first preprocessing the neighbors $N(x)$ in $O(\deg(x))$ time, and then
  generating the random edge to be contracted in $O(1)$
  time, e.g., see~\cite{BringmannP17}. The contraction can also implemented in
  $O(\deg(x))$ time, so together this gives us $O(\deg(x))$ time. 
\end{proof}

\begin{proof}[Proof of Theorem~\ref{lem:randomContract}]
By the construction of \RC, it follows that $T \subseteq V_H$. Moreover, since a
single random contraction preserves the Schur complement in expectation by
Lemma~\ref{lem:singleContract}, we get that our output sparsifier $H$
preserves $\SCf(\mat{M},T)$ in expectation in each iteration. Furthermore, the number of edges is
always upper bounded by $m$ because a contraction cannot increase the
number of edges.
	
%	To bound the number of edges, we start by observing that the algorithm maintains the invariant that the number of edges is at most $m$.
%	This is because it does not create any new edges, but instead only moves the endpoints of existing ones around.
	 %Notice that the invariant trivially holds before the first contraction since the input graph has $m$ edges. When performing the $k$-th contraction on a non-terminal vertex $x$, the algorithm removes exactly $\deg_{(k-1)}(x)$ vertices and adds at most $\deg_{(k-1)}(x)$, where $\deg_{(k-1)}(x)$ denotes the degree of $x$ in the graph $G_{(k-1)}$ obtained after performing $(k-1)$ contractions. Thus the number of multi-edges is at most $m$. 
	 
 Let $G_{(k)}$ denote the graph at the $k$-th iterative step in our algorithm,
 and denote by $\deg_{(k)}(x)$ the degree of $x$ in $G_{(k)}$.
  By Lemma~\ref{lem:singleContract}, the expected running time for removing a single
 non-terminal~$x$ via a reweighted random contraction in the graph $G_{(k)}$ is
 $O(\deg_{(k)}(x))$.
 We can implement a data structure using linked lists and bucketed vertex degrees
 to maintain and query the minimum degree vertex at each step
 in $O(\deg_{(k)}(x))$ time.
%  \footnote{The neighbors of $x$ not contracted with
% $x$ will decrease in degree by at most 1, while the neighbor $x$ is
% contracted to will increase in degree by at most $\deg_{(k)}(x)$. We then
% maintain a linked list of bucketed vertex degrees, and each neighbor of $x$
% not contracted with $x$ can only move to the next bucket taking $O(1)$ time,
% while the neighbor $x$ is contracted to will move at most $\deg_{(k)}(x)$
% buckets in $O(\deg_{(k)}(x))$ time.}
 At each iteration,
 the minimum degree vertex in $V \setminus T$ is~$x$ by construction.
 Since the number of edges throughout the procedure is at most $m$, it follows that
 $\deg_{(k)}(x) \le 2m/(n-k)$. Therefore, the
 total running time of \RC is bounded by $O(m
 \sum_{k=0}^{n - 1}\frac{1}{n-k}) = O(m \log n)$.
%We finally bound the overall running time of our procedure. By Lemma~\ref{lem:singleContract}, the expected running time for removing a single non-terminal $x$ via rewighted random contraction in the graph $G_{(k)}$ is $O(\deg_{(k)}(x))$. By construction, we have that $x$ is a minimum degree vertex in $V \setminus T$. Since we showed that the number of edges throughout the procedure is at most $m$, it follows that $\deg_{(k)}(x)) \leq O(m/(n'-k))$, where $n' = |V \setminus T|$. Thus the overall expected running time of \RC is bounded by $O(m \sum_{i=1}^{n'-1}\frac{1}{n'-i}) = O(m \log n)$.
\end{proof}
%\todo{Add running time -- Should be again bounded by $O(m)$ because single contraction costs only $\deg(x)$?}

In contrast, it is known that the \SC algorithm requires $\Omega(n^3)$
on \emph{almost all} sparse graphs~\cite{LiptonRT79}.
Even simply determining an elimination ordering of vertices with the minimum
degree at each iteration as in the \SC algorithm also requires at
least $\Omega(n^{2-\varepsilon})$ time, for all $\varepsilon > 0$,
under a well-accepted conditional hardness
assumption~\cite{cummings2019fast}.

%%% Local Variables:
%%% mode: latex
%%% TeX-master: "vertex-sparsification"
%%% End:

% !TEX root = main.tex
\section{Guarantees for Graph Embeddings}
\label{sec:Guarantees}

In this section, we give theoretical guarantees by proving that our two coarsening algorithms \SC and \RC preserve graph embeddings among terminal vertices. Let $G=(V,E,w)$ be an undirected, weighted graph whose node-embedding function we want to learn. Assume that the parameters associated with $G$ are geometrically decreasing,
i.e., $\theta_{i} = \theta^{i}$ and $\theta \in (0,1)$
where $i \in [W]$. While this version does not exactly match DeepWalk's setting
where all $\theta_i$ values are $1/10$,
it is a close approximation for most real-world graphs, as they are typically
expanders with low degrees of separation.

Our coarsening algorithm for graph embeddings first pre-processes $G$, by building a closely related graph $\widehat{G}$ that corresponds to the SDDM matrix $\mat{M} = \mat{D}-\theta \mat{A}$, and then runs \SC on top of $\widehat{G}$ with respect to  the terminals $T$. Let $H$ with $V(H) \supseteq T$ be the output graph and recall that its underlying matrix is SDDM, i.e., $\SCf(\mat{M},T) = \mat{D}'_H - \mat{A}_H$, where $\mat{D}_H' = \mat{D}_H + \mat{D}^{s}_H$. Below we define the graph embedding of $H$.

\vspace{0.25cm}
\begin{definition}[NetMFSC] Given a graph $G$, a target dimension $d$ and the graph $H$ defined above, the graph embedding \emph{NetMFSC} of $H$ is given by
\begin{align}
\label{eq:NetMFSC}
\textsc{SVD} \Bigg(
  \log^{+} \Bigg(
    \sum_{i = 1}^{W} \mat{D}_H'^{-1} & \left( \mat{A}_H  \mat{D}_H'^{-1} \right)^{i}
     + \mat{D}_H'^{-1} - \mat{D}^{-1}_{T,T}
  \Bigg) , d \Bigg).
\end{align}
\end{definition}

%Since the Schur complement modifies vertex degrees, we modify NetMF
%accordingly to allow for degree adjustments.
%Formally, 
%we define the updated graph embedding $\textsc{NetMFSC}\left(G, H, d, %(\theta_1,\theta_2,\dots,\theta_W)\right)$ to be
%\begin{align}
%\textsc{SVD} \left(
%  \textsc{trunc\_log} \left(
%    \sum_{i = 1}^{W} \theta_i D\left( H \right) ^{-1} \left( A\left( H %\right)  D\left( H \right) ^{-1} + D\left( H \right)^{-1} - D\left( G %\right)^{-1} \right)^{i}
%  \right), d \right).
%\label{eq:NetMF}
%\end{align}

%For the case where the coefficients are geometrically decreasing,
%we can show that the Schur complement preserves the embedding exactly.
%While this version is not exactly the same as deep walk, which chooses
%$\theta_{1} = \theta_{2} = \ldots = \theta_{10} = 1/10$, it is a close
%approximation on fast-mixing graphs, which include most real-world graphs.

%This is because on a fast mixing graph, $(AD^{-1})^{i}$ rapidly approaches
%the uniform graph, and for the choice of $\theta = (1 - \delta)$,
%the first $\delta^{-1}$ terms are all in the range of $[0.9, 1/e]$,
%which can be viewed as constants.
The lemma below shows that Schur complement $H$ together with \text{NetMFSC} exactly preserve the node embedding of the terminal vertices in the original graph $G$.

\vspace{0.25cm}
\begin{theorem}
	\label{lem:SCGuarantees}
  For any graph $G$, any subset of vertices~$T$, and any parameter $\theta \in (0,1)$,
  let the limiting \emph{NetMF} embedding with parameters $\theta_i$,
  for $i=1,\ldots,W$, be
	\[
	\mat{R}(G)
	:=
	\lim_{W \rightarrow \infty}
	\text{\emph{NetMF}}\left(G, d, \left(\theta^{i}\right)_{i=1}^{W} \right).
	\]
	For any threshold minimum degree $\Delta$, let $H = \SCf(\mat{M}, T, \Delta)$ with $\mat{M} = \mat{D} - \theta \mat{A}$ be the output graph along with its associated embedding
	\[
	\mat{R} (H)	:=
	\lim_{W \rightarrow \infty} 
	\text{\emph{NetMFSC}}\left(G, H, d\right).
	\]
	
Then we have that $\mat{R}(G)_{T,T} = \mat{R}(H)_{T,T}$ up to a rotation.	
	%Then for any vector on $T$, $\vec{y} \in \R^{T}$, we have
  %$\|\vec{y}^{\top}  R\left(G\right)_{T, :} \|_{2}
%	=
%	\|\vec{y}^{\top} R\left(H\right)_{T, :} \|_{2}.$
\end{theorem}

An important ingredient needed to prove the above lemma is the following fact.

\vspace{0.25cm}
\begin{fact}
	\label{fact:SchurInverse}
%	For any invertible matrix $M$ and any subset of indices $T$, we have
%	\[
%	\SC\left(M, T\right)^{-1}
%	=
%	\left(M^{-1} \right)_{T, T}
%	\]
	If $\mat{M}$ is an invertible matrix and $T \subseteq V$, it holds that
	$\SCf(\mat{M}, T)^{-1} = \mat{M}^{-1}_{T, T}$.
\end{fact}

\begin{proof}[Proof of Theorem~\ref{lem:SCGuarantees}]
	Recall that the \text{NetMF} of $G$ is the
	SVD factorization of an entry-wise truncated logarithm of the random walk matrix
	$\sum_{i = 1}^{W} \theta_{i} \mat{D}^{-1} ( \mat{A} \mat{D}^{-1})^{i}$. Substituting in our choices of $\theta_i = \theta^{i}$ and since $W$ tends to infinity, we get that this matrix is the inverse of the SDDM matrix $\mat{D} - \theta \mat{A}$. Concretely, we have
	\begin{align}
	\label{eq: infG}
	\lim_{W \rightarrow \infty}
	\sum_{i = 1}^{W} \theta^{i}  \left( \mat{A} \mat{D}^{-1} \right)^{i}
	& = \mat{D}^{-1} \sum_{i=0}^{\infty} \left(\theta \mat{A} \mat{D}^{-1}\right)^i - \mat{D}^{-1} \nonumber \\
	& = \mat{D}^{-1} \left(\mat{I} - \theta \mat{A} \mat{D}^{-1}\right)^{-1} - \mat{D}^{-1} \nonumber \\
	& = \left (\mat{D} - \theta \mat{A} \right)^{-1} - \mat{D}^{-1}.
	\end{align}
	By Fact~\ref{fact:SchurInverse}, we know that the Schur complement
	exactly preserves the entries among vertices in $T$ in the inverse, i.e., \begin{equation}  
\label{eq: terminalInverse}	
\SCf\left(\mat{D}-\theta \mat{A},T \right)^{-1} = \left(\mat{D}-\theta \mat{A} \right)^{-1}_{T,T}. 
\end{equation}
	Furthermore, the definition of \text{NetMFSC} in Eq.~(\ref{eq:NetMFSC}) performs diagonal adjustments and thus ensures that the matrices being factorized are exactly the same. Formally, we have
	\begin{align}
	\label{eq: infH}
	 \lim_{W \rightarrow \infty} \sum_{i = 1}^{W} \mat{D}_H'^{-1} \left( \mat{A}_H  \mat{D}_H'^{-1} \right)^{i}   + \mat{D}_H'^{-1} - \mat{D}^{-1}_{T,T}
	 &  =  \mat{D}_H'^{-1}\sum_{i = 0}^{\infty} \left(\mat{A}_H \mat{D}_H'^{-1} \right)^i - \mat{D}^{-1}_{T,T} \nonumber \\
	 & = \mat{D}_H'^{-1}(\mat{I} -\mat{A}_H \mat{D}_H'^{-1})^{-1} - \mat{D}^{-1}_{T,T} \nonumber \\
	 & = \SCf(\mat{D}- \theta \mat{A},T)^{-1} - \mat{D}^{-1}_{T,T} \nonumber  \\
    & \hspace{-0.22cm}\stackrel{\text{Eq. }(\ref{eq: terminalInverse})}{=} (\mat{D} - \theta \mat{A})_{T,T}^{-1} - \mat{D}^{-1}_{T,T}.
    \end{align}
Since the matrices in Eq.~(\ref{eq: infG}) and~(\ref{eq: infH}) are the same when restricted to the terminal set $T$, we get that their factorizations are also the same up to a rotation, which in turn implies that $\mat{R}(G)_{T,T} = \mat{R}(H)_{T,T}$ up to rotation. \qedhere
%\gramoz{Below you say that $d$ must go to infinity? Sushant and I could not see why you need that? Please explain...}
%
% 
%    Furthermore, as $d \rightarrow \infty$, the factorization exactly
%	preserves the entries of these matrices.
%	That is, for $G$, we have
%	$
%	 R(G)_{T, :}^{\top}
%	=
%	(D( G ) - \theta A( G ) )^{-1}_{T, T} - D( G )_{T, T},
%	$
%	and in turn
%	\[
%	\left\| \vec{y}^{\top}  R\left(G\right)_{T, :} \right\|_{2}^{2}
%	=
%	\vec{y}^{\top}
%	\left( D\left(G\right) - \theta A\left(G\right) \right)^{-1}_{T, T}
%	\vec{y}
%	- \vec{y}^{\top} D\left(G\right) \vec{y}.
%	\]
%	So after the diagonal adjustment %\todo{is this formal enough?},
%	the quantity we're seeking to preserve is exactly the
%	quadratic form against this subset of the inverse matrix.
%	Therefore, the result follows from Fact~\ref{fact:SchurInverse}.
\end{proof}

\section{Experiments}
\label{sec:Experiments}

In this section, we investigate how the vertex sparsifiers \SC and \RC affect
the predictive performance of graph embeddings for two different learning tasks.
Our multi-label vertex classification experiment builds on the framework
for NetMF~\citep{QiuDMLWT18}, and evaluates the accuracy of
logistic regression models that use graph embeddings obtained by first
coarsening the networks.
Our link prediction experiment
builds on the setup in node2vec~\citep{GroverL2016},
and explores the effect of vertex sparsification on AUC
scores for several popular link prediction baselines.

\subsection{Multi-label Vertex Classification}

\paragraph{Datasets.}
The networks we consider and their statistics are listed in Table~\ref{table:datasets}.
\blogcatalog~\citep{tang2009relational} models the social relationships
of online bloggers, and its vertex labels represent topic categories of the authors.
\flickr~\citep{tang2009relational} is a network of user contacts on
the image-sharing website Flickr, and its labels represent groups interested
in different types of photography.
\youtube
\citep{yang2015defining} is a social network on users of the popular
video-sharing website, and its labels are user-defined groups with
mutual interests in video genres.
We only consider the largest connected component of the \youtube network.

{
\vspace{-0.4cm}
\setlength{\tabcolsep}{4pt} % Default value: 6pt
\renewcommand{\arraystretch}{1} % Default value: 1
\begin{table}[H]
	\caption{Statistics of the networks in our vertex classification experiments.}
	\label{table:datasets}
  \vspace{0.2cm}
	\centering
	\begin{tabular}{lrrrrrrrrr}
		\toprule
    Dataset & Nodes & Edges & Classes & Labels \\
		\midrule
    \blogcatalog & 10,312 & 333,983 & 3,992 & 14,476 \\
    \flickr & 80,513 & 5,899,882 & 195 & 107,741 \\
    \youtube & 1,134,890 & 2,987,624 & 47 & 50,669 \\
    \vspace{-1.00cm}
	\end{tabular}
\end{table}
}

\paragraph{Evaluation Methods.}
We primarily use the embedding algorithm NetMF~\citep{QiuDMLWT18}, which unifies
LINE~\citep{tang2015line} and DeepWalk~\citep{PerozziAS14} via a matrix
factorization framework. LINE corresponds to NetMF when the window size equals 1,
and DeepWalk corresponds to NetMF when the window size is greater than 1.
We use the one-vs-all logistic
regression model implemented in scikit-learn~\citep{JMLR:Pedregosa2011} to 
investigate the quality of our vertex sparsifiers
for the multi-label vertex classification task. 
For each dataset and embedding algorithm, we compute the embeddings
of the original network and the two sparsified networks given by \SC and \RC.
Then for each of these embeddings, we evaluate the model
at increasing training ratios using the prediction pipeline in the NetMF
experiments~\citep{QiuDMLWT18}.

Since all of the nodes in \blogcatalog and \flickr are labeled, we downsample
the training set by randomly selecting half of the vertices and completely
discarding their labels. This induces a smaller label set which we use for 
both training and evaluation. The \youtube network is already sparsely labeled,
so we do not modify its training set.
We refer to the labeled vertices as terminals and prohibit the
sparsification algorithms from eliminating these nodes.
In all of our experiments, we use the minimum degree threshold $\Delta = 30$
for the \SC and \RC algorithms.
We choose the conventional target dimension of $d=128$ for all graph embeddings.
For LINE embeddings, we run NetMF with window size $W=1$.
For DeepWalk embeddings, we run \NetMF with the window size $W=10$ in the \blogcatalog and \flickr experiments,
and we use the window size $W = 2$ for the \youtube network
because of the density of the resulting random walk matrix.
To further study DeepWalk embeddings for the YouTube network,
we compare our results with the
novel embedding algorithm NetSMF~\citep{QiuDMLWWT19}, which is leverages
an intermediate spectral sparsifier for the
dense random walk matrix.
We use the window size $W=10$ in all instances of NetSMF.

For the \blogcatalog experiments, we vary the training ratio from $10\%$ to $90\%$,
and for \flickr and \youtube we vary the training ratio from $1\%$ to $10\%$.
In all instances, we perform 10-fold cross validation
and evaluate the prediction accuracy in terms of the mean Micro-F1 and Macro-F1 scores.
All of our experiments are performed on a Linux virtual machine with
64 Xeon E5-2673 v4 virtual CPUs (2.30GHz),
432GB of memory,
and a 1TB hard disk drive.

\begin{figure*}
  \centering
  \includegraphics[width=0.9\textwidth]{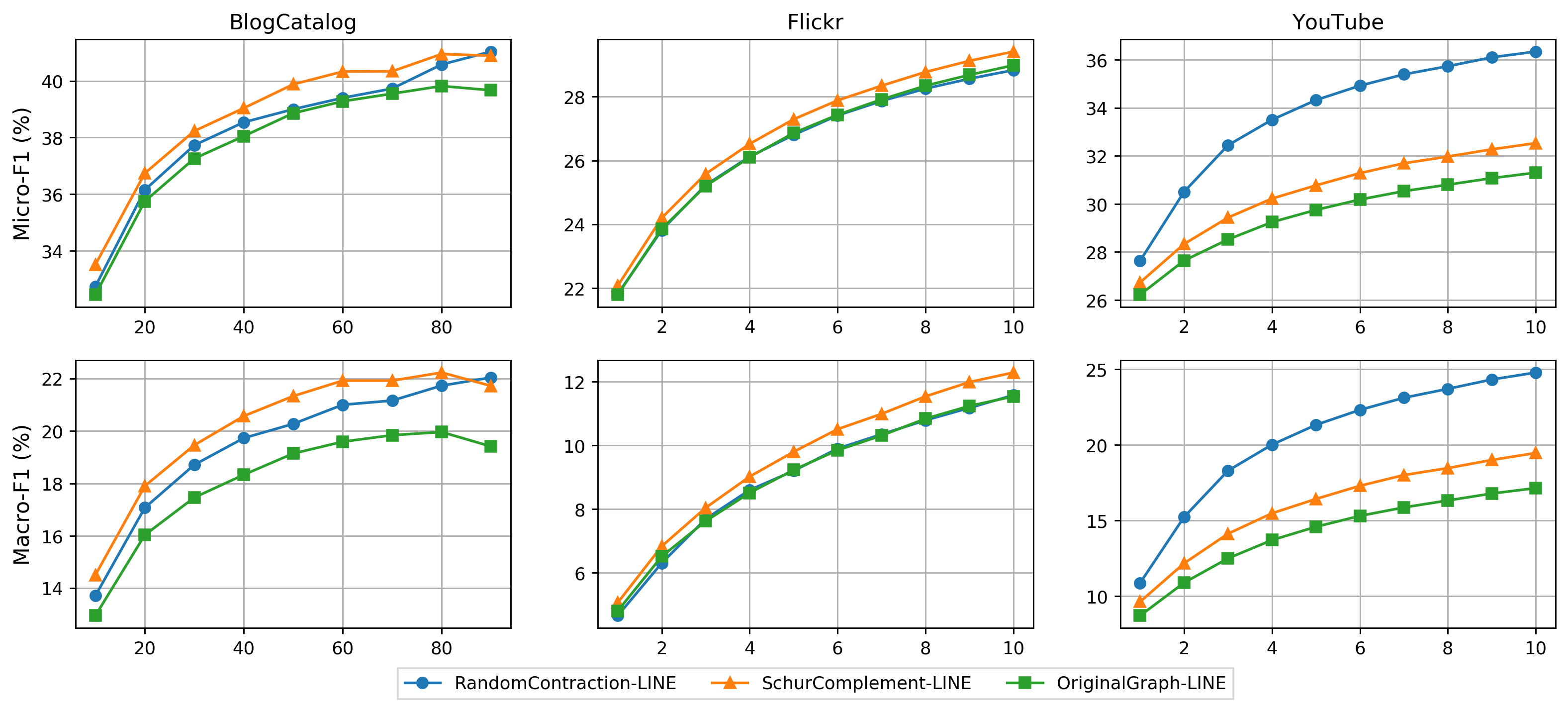}
  %\vspace{-0.4cm}
  \caption{Accuracy of sparsified \LINE embeddings.
  The $x$-axis denotes the training ratio~$(\%)$
  and the $y$-axis in the top and bottom row denotes the mean
  Micro-F1 and Macro-F1 scores, respectively.}
  \label{fig:f1-line}
  \vspace{-0.25cm}
\end{figure*}

\paragraph{Results for \LINE (\NetMF with $W=1$).}
We start by evaluating the classification performance
of LINE embeddings of the sparsified networks relative to the
originals and plot the results across all datasets in Figure~\ref{fig:f1-line}.
Our first observation is that quality of the embedding for
classification always improves by running the \SC sparsifier.
To explain this phenomenon, we note that \LINE computes an
embedding using length $W=1$ random walks. This approach, however, is often
inferior to methods that use longer random walks such as \deepwalk.
When eliminating vertices of a graph using the Schur complement, the resulting
graph perfectly preserves random walk transition probabilities through the
eliminated vertex set with respect to the original graph.
Thus, the \LINE embedding of a graph sparsified using \SC implicitly
captures longer length random walks through low-degree vertices and
hence more structure of the network.
For the \youtube experiment, we observe that \RC substantially outperforms
\SC and the baseline \LINE embedding. We attribute this behavior to the fact that
contractions preserve edge sparsity unlike Schur complements.
It follows that \RC typically eliminates more nodes than \SC when given a
degree threshold.
In this instance, the \youtube network sparsified by \SC
has 84,371 nodes while \RC produces a network with 53,291 nodes
(down from 1,134,890 in the original graph).

\begin{figure*}
  \centering
  \includegraphics[width=0.9\textwidth]{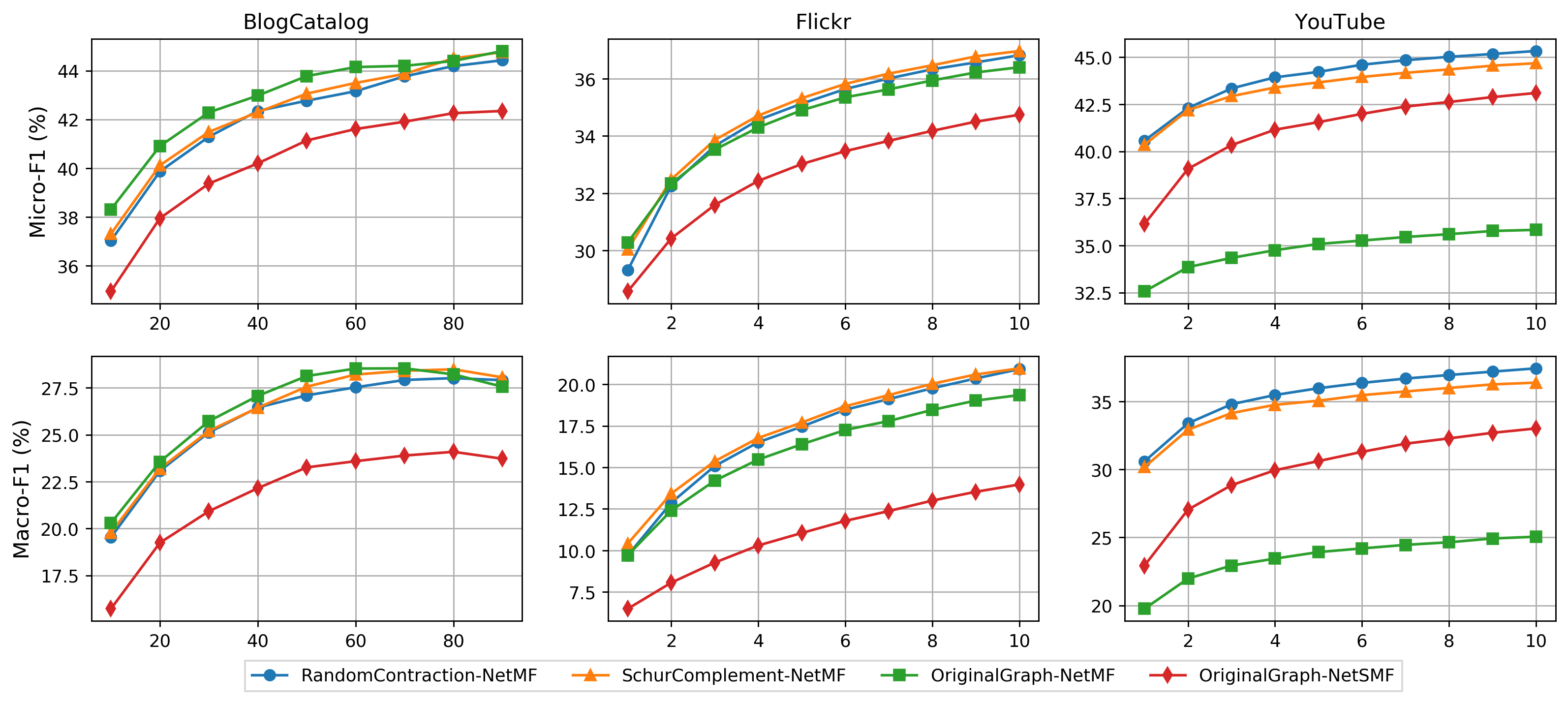}
  %\vspace{-0.4cm}
  \caption{Accuracy of sparsified \deepwalk embeddings.
  The $x$-axis denotes the training ratio~$(\%)$
  and the $y$-axis in the top and bottom row denotes the mean
  Micro-F1 and Macro-F1 scores, respectively.}
  \label{fig:f1-netmf}
\end{figure*}
%\vspace{-0.6cm}

\paragraph{Results for \deepwalk (\NetMF with $W \ge 2$).}
Now we consider the same
classification experiments using \deepwalk and NetSMF embeddings. 
We plot the prediction performance for various
training ratios across all datasets in Figure~\ref{fig:f1-netmf}.
Again, we observe that the vertex-sparsified embeddings perform at least as well as the
embedding of the original graph for this multi-label classification task, which
we attribute to the implicit use of longer random walks.
In the \youtube experiment we observe a dramatic improvement over the
baseline, but this is because of an entirely different reason than before.
A core subroutine of \NetMF with $W \ge 2$ is computing the truncated SVD of a
dense random walk matrix graph, so the benefits of vertex
sparsification surface in two ways.
First, the bottleneck in the runtime of the classification pipeline is
the truncated SVD. By preprocessing the graph to reduce its vertex set we
noticeably speed up this SVD call
(e.g., see YouTube and NetMF in Table~\ref{table:runtimes}).
%For example, in the \youtube experiment the times needed to compute the \deepwalk
%embedding for \RC, \SC and the original graph were
%15m9s, 34m13s, and 78m34s, respectively (\Cref{table:runtimes}). The time to run both \RC and \SC is under one minute.
Second, the convergence rate of the approximate
SVD depends on the dimension of the underlying matrix, so the sparsified graphs
lead to more accurate eigendecompositions and hence higher quality embeddings.
We reiterate that for the \youtube experiment, we set $W=2$ to meet a 432GB
memory constraint whereas in the \deepwalk experiments in~\citet{PerozziAS14}
the authors set $W=10$.
We also run NetSMF with $W=10$ on the original graphs as another benchmark,
but in some instances we need to use fewer samples to satisfy our memory limit,
hence the lower accuracies than in~\citet{QiuDMLWWT19}.
When trained on $10\%$ of the label set, \SC achieves
24.45\% and 44.67\% relative gains over \deepwalk in terms of Micro-F1 and
Macro-F1 scores. Furthermore, since \RC yields a coarser graph on fewer nodes, it
gives 26.43\% and 48.94\% improvements relative to \deepwalk.

{
\setlength{\tabcolsep}{3.8pt} % Default value: 6pt
\renewcommand{\arraystretch}{1} % Default value: 1
\begin{table}[H]
  %\vspace{-0.4cm}
  \caption{Running times of the coarsening and graph embedding stages in
  the vertex classification experiment (seconds).}
  \vspace{0.2cm}
	\label{table:runtimes}
	\centering
	\begin{tabular}{lrrrrrrrrr}
		\toprule
    Network & Sparsify & LINE & NetMF & NetSMF \\
		\midrule
    \blogcatalog & -- & 3.72 & 20.80 & 45.56 \\
    \blogcatalog SC & 1.53 & 3.88 & 15.66 & -- \\
    \blogcatalog RC & 2.47 & 3.83 & 15.61 & -- \\
    \hline
    \flickr    & -- & 51.62 & 1,007.17 & 950.60 \\
    \flickr SC & 19.43 & 56.62 & 571.70 & -- \\
    \flickr RC & 59.43 & 43.41 & 597.08 & -- \\
    \hline
    \youtube   & -- & 147.55 & 4,714.88 & 3,458.75 \\
    \youtube SC & 22.91 & 44.84 & 2,053.59 & -- \\
    \youtube RC & 44.13 & 16.55 & 909.16 & -- \\
    \vspace{-0.75cm}
	\end{tabular}
\end{table}
}

\subsection{Link Prediction}

\paragraph{Datasets.}

We build on the experimental framework
in node2vec~\citep{GroverL2016}
and evaluate our vertex sparsification algorithms on the following datasets:
Facebook~\citep{snapnets},
arXiv ASTRO-PH~\citep{snapnets},
and Protein-Protein Interaction (PPI)~\citep{stark2006biogrid}.
We present the statistics of these networks in Table~\ref{table:link-datasets}.
Facebook is a social network where nodes represent users
and edges represent a friendship between two users.
The arXiv graph is a collaboration network generated from papers submitted
to arXiv. Nodes represent scientists and an edge is present between two
scientists if they have coauthored a paper.
In the PPI network for Homo Sapiens,
nodes represent proteins and edges
indicate a biological interaction between a pair of proteins.
We consider the largest connected component of the arXiv and PPI graphs.

{
\vspace{-0.4cm}
\setlength{\tabcolsep}{4pt} % Default value: 6pt
\renewcommand{\arraystretch}{1} % Default value: 1
\begin{table}[H]
	\caption{Statistics of the networks in our link prediction experiments.}
	\label{table:link-datasets}
  \vspace{0.2cm}
	\centering
	\begin{tabular}{lrrrrrrrrr}
		\toprule
    Dataset & Nodes & Edges \\
		\midrule
    Facebook & 4,039 & 88,234 \\
    arXiv ASTRO-PH & 17,903 & 196,972 \\
    Protein-Protein Interaction (PPI) & 21,521 & 338,625 \\
    \vspace{-0.75cm}
	\end{tabular}
\end{table}
}

\paragraph{Evaluation Methods.}
In the \emph{terminal} link prediction task,
we are given a graph and a set of terminal nodes.
A subset of the terminal-to-terminal edges are removed,
and the goal is to accurately predict edges and non-edges
between terminal pairs in the original graph.
We generate the labeled dataset of edges as follows:
first, randomly select a subset of terminal nodes;
to obtain positive examples, remove 50\% of the edges chosen
uniformly at random between terminal nodes
while ensuring that the network is still connected;
to obtain negative examples, randomly choose an equal number of
terminal pairs that are not adjacent in the original graph.
We select 500 (12.4\%) nodes as terminals in the Facebook network,
2000 (9.3\%) in PPI, and 4000 (22.3\%) in arXiv.

For each network and terminal set, we use \RC and \SC to coarsen the graph.
Then we compute node embeddings using \LINE and \NetMF.
We calculate an embedding for each edge $(u,v)$ by taking the Hadamard
or weighted L2 product of the node embeddings for $u$ and $v$~\citep{GroverL2016}.
Finally, we train a logistic regression model using the edge embeddings as
features, and report the area under the receiver operating characteristic curve
(AUC) from the prediction scores.

{
%\vspace{-0.4cm}
\setlength{\tabcolsep}{3.8pt} % Default value: 6pt
\renewcommand{\arraystretch}{1} % Default value: 1
\begin{table}
  \caption{Area under the curve (AUC) scores for
  different operators,
  coarsening, and embedding algorithms for the link prediction task.}
	\label{table:auc_scores}
  \vspace{0.2cm}
	\centering
	\begin{tabular}{llrrrrrrrr}
		\toprule
    Operator & Algorithm & Facebook & arXiv & PPI\\
		\midrule
    & LINE & 0.9891 & 0.9656 & 0.9406 \\
    & RC + LINE & 0.9937 & 0.9778 & \textbf{0.9431} \\
    Hadamard & SC \hspace{0.03cm}+ LINE & \textbf{0.9950} & \textbf{0.9854} & 0.9418 \\
    & NetMF & 0.9722 & 0.9508 & 0.8558 \\
    & RC + NetMF & 0.9745 & 0.9752 & 0.9072 \\
    & SC \hspace{0.03cm}+ NetMF & 0.9647 & 0.9811 & 0.9018 \\
		\midrule
    & LINE & 0.9245 & 0.6129 & 0.7928 \\
    & RC + LINE & 0.9263 & 0.6217 & 0.7983 \\
    Weighted L2 & SC \hspace{0.03cm}+ LINE & 0.9523 & 0.6824 & 0.7835 \\
    & NetMF & 0.9865 & 0.9574 & 0.8646 \\
    & RC + NetMF & 0.9852 & 0.9800 & 0.9207 \\
    & SC \hspace{0.03cm}+ NetMF & 0.9865 & 0.9849 & 0.9120 \\
    \vspace{-1.00cm}
	\end{tabular}
\end{table}
}

\paragraph{Results.}
We summarize our results in Table~\ref{table:auc_scores}. For all of the 
datasets, using \RC or \SC for coarsening and \LINE with Hadamard products for
edge embeddings gives the best results.
Moreover, coarsening consistently outperforms the baseline
(i.e., the same network and terminals without any sparsification).
We attribute the success of \LINE-based embeddings in this experiment
to the fact that our coarsening algorithms preserve random walks through the
eliminated nodes; hence, running \LINE on a coarsened graph implicitly uses
longer-length random walks to compute embeddings.
We see the same behavior with coarsening and \NetMF, but the resulting AUC
scores are marginally lower.
Lastly, our experiments also highlight the importance of choosing the right
binary operator for a given node embedding algorithm.
%(e.g., using the weighted L2 operator
%to combine \NetMF node embeddings instead of the Hadamard product).

%%% Local Variables:
%%% mode: latex
%%% TeX-master: "vertex-sparsification"
%%% End:

\section{Conclusion}
We introduce two vertex sparsification algorithms based on
Schur complements to be used as a preprocessing routine when computing
graph embeddings of large-scale networks.
Both of these algorithms repeatedly choose a vertex to remove and add
new edges between its neighbors.
In Section~\ref{sec:Guarantees} we demonstrate that these algorithms exhibit provable trade-offs
between their running time and approximation quality.
The \RC algorithm is faster because it contracts the eliminated
vertex with one of its
neighbors and reweights all of the edges in its neighborhood,
while the \SC algorithm adds a
weighted clique between all pairs of neighbors of the eliminated vertex
via Gaussian elimination.
We prove that the random contraction based-scheme produces a graph that is the
same in expectation as the one given by Gaussian elimination, which in turn
yields the matrix factorization that random walk-based graph embeddings
such as DeepWalk, NetMF and NetSMF aim to approximate.

The main motivation for our techniques is that Schur complements preserve
random walk transition probabilities through eliminated vertices, which we
can then exploit by factorizing smaller matrices on the terminal set of vertices.
We demonstrate on commonly-used benchmarks for graph embedding-based
multi-label vertex classification tasks
that both of these algorithms empirically improve the
prediction accuracy compared to using
graph embeddings of the original and unsparsified networks,
while running in less time and using substantially less memory.

%%% Local Variables:
%%% mode: latex
%%% TeX-master: "vertex-sparsification"
%%% End:

\vspace{-0.30cm}
\section*{Acknowledgements}
MF did part of this work while supported by
an NSF Graduate Research Fellowship under grant DGE-1650044 at the Georgia
Institute of Technology.
SS and GG are partly supported by an NSERC Discovery grant awarded to
SS by NSERC (Natural Sciences and Engineering Research Council of Canada).
RP did part of this work while at Microsoft Research Redmond,
and is partially supported by the NSF under grants CCF-1637566 and CCF-1846218.

\vspace{-0.30cm}
\bibliography{references}

\begin{thebibliography}{}

\bibitem[Abu-El-Haija et~al., 2019]{pmlr-v97-abu-el-haija19a}
Abu-El-Haija, S., Perozzi, B., Kapoor, A., Alipourfard, N., Lerman, K.,
  Harutyunyan, H., Steeg, G.~V., and Galstyan, A. (2019).
\newblock {M}ix{H}op: Higher-order graph convolutional architectures via
  sparsified neighborhood mixing.
\newblock In {\em Proceedings of the 36th International Conference on Machine
  Learning}, pages 21--29. PMLR.

\bibitem[Batson et~al., 2013]{BatsonSST13}
Batson, J., Spielman, D.~A., Srivastava, N., and Teng, S.-H. (2013).
\newblock Spectral sparsification of graphs: {T}heory and algorithms.
\newblock {\em Communications of the ACM}, 56(8):87--94.

\bibitem[Benson and Kleinberg, 2019]{BensonK18:arxiv}
Benson, A. and Kleinberg, J. (2019).
\newblock Link prediction in networks with core-fringe data.
\newblock In {\em Proceedings of the 28th International Conference on World
  Wide Web}, pages 94--104. ACM.

\bibitem[Borgatti and Everett, 2000]{BorgattiE00}
Borgatti, S.~P. and Everett, M.~G. (2000).
\newblock Models of core/periphery structures.
\newblock {\em Social networks}, 21(4):375--395.

\bibitem[Briggs et~al., 2000]{BriggsHM00:book}
Briggs, W.~L., Henson, V.~E., and McCormick, S.~F. (2000).
\newblock {\em A Multigrid Tutorial}.
\newblock SIAM.

\bibitem[Bringmann and Panagiotou, 2017]{BringmannP17}
Bringmann, K. and Panagiotou, K. (2017).
\newblock Efficient sampling methods for discrete distributions.
\newblock {\em Algorithmica}, 79(2):484--508.

\bibitem[Bruna et~al., 2014]{BrunaZSL14}
Bruna, J., Zaremba, W., Szlam, A., and LeCun, Y. (2014).
\newblock Spectral networks and locally connected networks on graphs.
\newblock In {\em Proceedings of the 2nd International Conference on Learning
  Representations}.

\bibitem[Cao et~al., 2015]{Cao2015grarep}
Cao, S., Lu, W., and Xu, Q. (2015).
\newblock {G}ra{R}ep: {L}earning graph representations with global structural
  information.
\newblock In {\em Proceedings of the 24th ACM International Conference on
  Information and Knowledge Management}, pages 891--900.

\bibitem[Chen et~al., 2018]{ChenPHS18}
Chen, H., Perozzi, B., Hu, Y., and Skiena, S. (2018).
\newblock {HARP}: {H}ierarchical representation learning for networks.
\newblock In {\em Proceedings of the Thirty-Second AAAI Conference on
  Artificial Intelligence}, pages 2127--2134.

\bibitem[Cheng et~al., 2015]{ChengCLPT15-colt}
Cheng, D., Cheng, Y., Liu, Y., Peng, R., and Teng, S.-H. (2015).
\newblock Efficient sampling for {G}aussian graphical models via spectral
  sparsification.
\newblock In {\em Proceedings of The 28th Conference on Learning Theory
  \textnormal{(COLT)}}, pages 364--390. PMLR.

\bibitem[Chevalier and Safro, 2009]{chevalier2009comparison}
Chevalier, C. and Safro, I. (2009).
\newblock Comparison of coarsening schemes for multilevel graph partitioning.
\newblock In {\em International Conference on Learning and Intelligent
  Optimization}, pages 191--205. Springer.

\bibitem[Cummings et~al., 2019]{cummings2019fast}
Cummings, R., Fahrbach, M., and Fatehpuria, A. (2019).
\newblock A fast minimum degree algorithm and matching lower bound.
\newblock {\em arXiv preprint arXiv:1907.12119}.

\bibitem[Englert et~al., 2014]{EnglertGKRTT14}
Englert, M., Gupta, A., Krauthgamer, R., R{\"{a}}cke, H., Talgam{-}Cohen, I.,
  and Talwar, K. (2014).
\newblock Vertex sparsifiers: New results from old techniques.
\newblock {\em {SIAM} J. Comput.}, 43(4):1239--1262.

\bibitem[Fahrbach et~al., 2018]{fahrbach2018graph}
Fahrbach, M., Miller, G.~L., Peng, R., Sawlani, S., Wang, J., and Xu, S.~C.
  (2018).
\newblock Graph sketching against adaptive adversaries applied to the minimum
  degree algorithm.
\newblock In {\em 2018 IEEE 59th Annual Symposium on Foundations of Computer
  Science (FOCS)}, pages 101--112. IEEE.

\bibitem[George and Liu, 1989]{george1989evolution}
George, A. and Liu, J. W.~H. (1989).
\newblock The evolution of the minimum degree ordering algorithm.
\newblock {\em SIAM Review}, 31(1):1--19.

\bibitem[Grover and Leskovec, 2016]{GroverL2016}
Grover, A. and Leskovec, J. (2016).
\newblock Node2vec: {S}calable feature learning for networks.
\newblock In {\em Proceedings of the 22nd ACM SIGKDD International Conference
  on Knowledge Discovery and Data Mining}, page 855–864.

\bibitem[Hamilton et~al., 2017]{HamiltonYL2017}
Hamilton, W.~L., Ying, Z., and Leskovec, J. (2017).
\newblock Inductive representation learning on large graphs.
\newblock In {\em Advances in Neural Information Processing Systems}, pages
  1024--1034.

\bibitem[Jia and Benson, 2019]{JiaB19}
Jia, J. and Benson, A.~R. (2019).
\newblock Random spatial network models for core-periphery structure.
\newblock In {\em Proceedings of the Twelfth {ACM} International Conference on
  Web Search and Data Mining}, pages 366--374. ACM.

\bibitem[Karypis and Kumar, 1998]{Karypis98}
Karypis, G. and Kumar, V. (1998).
\newblock A software package for partitioning unstructured graphs, partitioning
  meshes, and computing fill-reducing orderings of sparse matrices.

\bibitem[Kipf and Welling, 2017]{kipf2016semi}
Kipf, T.~N. and Welling, M. (2017).
\newblock Semi-supervised classification with graph convolutional networks.
\newblock In {\em Proceedings of the 5th International Conference on Learning
  Representations}.

\bibitem[Kyng et~al., 2016]{KyngLPSS16}
Kyng, R., Lee, Y.~T., Peng, R., Sachdeva, S., and Spielman, D.~A. (2016).
\newblock Sparsified cholesky and multigrid solvers for connection laplacians.
\newblock In {\em Proceedings of the 48th Annual {ACM} {SIGACT} Symposium on
  Theory of Computing (STOC)}, pages 842--850.

\bibitem[Kyng and Sachdeva, 2016]{KyngS16}
Kyng, R. and Sachdeva, S. (2016).
\newblock Approximate gaussian elimination for laplacians - fast, sparse, and
  simple.
\newblock In {\em 2016 IEEE 57th Annual Symposium on Foundations of Computer
  Science (FOCS)}, pages 573--582.

\bibitem[Leskovec and Krevl, 2014]{snapnets}
Leskovec, J. and Krevl, A. (2014).
\newblock {SNAP Datasets}: {Stanford} large network dataset collection.
\newblock \url{http://snap.stanford.edu/data}.

\bibitem[Liang et~al., 2018]{LiangGP18}
Liang, J., Gurukar, S., and Parthasarathy, S. (2018).
\newblock {MILE}: A multi-level framework for scalable graph embedding.
\newblock {\em arXiv preprint arXiv:1802.09612}.

\bibitem[Lipton et~al., 1979]{LiptonRT79}
Lipton, R.~J., Rose, D.~J., and Tarjan, R.~E. (1979).
\newblock Generalized nested dissection.
\newblock {\em SIAM Journal on Numerical Analysis}, 16(2):346--358.

\bibitem[Liu et~al., 2016]{LiuDSK16}
Liu, Y., Dighe, A., Safavi, T., and Koutra, D. (2016).
\newblock A graph summarization: {A} survey.
\newblock {\em CoRR}, abs/1612.04883.

\bibitem[Loukas and Vandergheynst, 2018]{LoukasV18}
Loukas, A. and Vandergheynst, P. (2018).
\newblock Spectrally approximating large graphs with smaller graphs.
\newblock In {\em Proceedings of the 35th International Conference on Machine
  Learning}, pages 3237--3246. PMLR.

\bibitem[Moitra, 2009]{Moitra09}
Moitra, A. (2009).
\newblock Approximation algorithms for multicommodity-type problems with
  guarantees independent of the graph size.
\newblock In {\em 50th Annual {IEEE} Symposium on Foundations of Computer
  Science (FOCS)}, pages 3--12.

\bibitem[Pedregosa et~al., 2011]{JMLR:Pedregosa2011}
Pedregosa, F., Varoquaux, G., Gramfort, A., Michel, V., Thirion, B., Grisel,
  O., Blondel, M., Prettenhofer, P., Weiss, R., Dubourg, V., et~al. (2011).
\newblock Scikit-learn: {M}achine learning in {P}ython.
\newblock {\em Journal of Machine Learning Research}, 12:2825--2830.

\bibitem[Perozzi et~al., 2014]{PerozziAS14}
Perozzi, B., Al-Rfou, R., and Skiena, S. (2014).
\newblock Deepwalk: {O}nline learning of social representations.
\newblock In {\em Proceedings of the 20th ACM SIGKDD International Conference
  on Knowledge Discovery and Data Mining}, pages 701--710. ACM.

\bibitem[Qiu et~al., 2019]{QiuDMLWWT19}
Qiu, J., Dong, Y., Ma, H., Li, J., Wang, C., Wang, K., and Tang, J. (2019).
\newblock Net{SMF}: {L}arge-scale network embedding as sparse matrix
  factorization.
\newblock In {\em Proceedings of the 28th International Conference on World
  Wide Web}, pages 1509--1520. ACM.

\bibitem[Qiu et~al., 2018]{QiuDMLWT18}
Qiu, J., Dong, Y., Ma, H., Li, J., Wang, K., and Tang, J. (2018).
\newblock Network embedding as matrix factorization: Unifying deepwalk, line,
  pte, and node2vec.
\newblock In {\em Proceedings of the Eleventh ACM International Conference on
  Web Search and Data Mining}, pages 459--467.

\bibitem[Stark et~al., 2006]{stark2006biogrid}
Stark, C., Breitkreutz, B.-J., Reguly, T., Boucher, L., Breitkreutz, A., and
  Tyers, M. (2006).
\newblock {B}io{GRID}: {A} general repository for interaction datasets.
\newblock {\em Nucleic Acids Research}, 34:D535--D539.

\bibitem[Tang et~al., 2015]{tang2015line}
Tang, J., Qu, M., Wang, M., Zhang, M., Yan, J., and Mei, Q. (2015).
\newblock {LINE}: {L}arge-scale information network embedding.
\newblock In {\em Proceedings of the 24th International Conference on World
  Wide Web}, pages 1067--1077.

\bibitem[Tang and Liu, 2009]{tang2009relational}
Tang, L. and Liu, H. (2009).
\newblock Relational learning via latent social dimensions.
\newblock In {\em Proceedings of the 15th ACM SIGKDD International Conference
  on Knowledge Discovery and Data Mining}, pages 817--826. ACM.

\bibitem[Tang and Liu, 2011]{Tang2011}
Tang, L. and Liu, H. (2011).
\newblock Leveraging social media networks for classification.
\newblock {\em Data Mining and Knowledge Discovery}, 23(3):447--478.

\bibitem[Thorup and Zwick, 2005]{ThorupZ05}
Thorup, M. and Zwick, U. (2005).
\newblock Approximate distance oracles.
\newblock {\em J. {ACM}}, 52(1):1--24.

\bibitem[Vattani et~al., 2011]{vattani2011preserving}
Vattani, A., Chakrabarti, D., and Gurevich, M. (2011).
\newblock Preserving personalized pagerank in subgraphs.
\newblock In {\em Proceedings of the 28th International Conference on Machine
  Learning}, page 793–800.

\bibitem[Veli{\v{c}}kovi{\'c} et~al., 2018]{velickovic18}
Veli{\v{c}}kovi{\'c}, P., Cucurull, G., Casanova, A., Romero, A., Lio, P., and
  Bengio, Y. (2018).
\newblock Graph attention networks.
\newblock In {\em Proceedings of the 6th International Conference on Learning
  Representations}.

\bibitem[Viswanathan et~al., 2019]{viswanathan2019}
Viswanathan, K., Sachdeva, S., Tomkins, A., and Ravi, S. (2019).
\newblock Improved semi-supervised learning with multiple graphs.
\newblock In {\em The 22nd International Conference on Artificial Intelligence
  and Statistics}, pages 3032--3041.

\bibitem[Wagner et~al., 2018]{WagnerGKM18}
Wagner, T., Guha, S., Kasiviswanathan, S., and Mishra, N. (2018).
\newblock Semi-supervised learning on data streams via temporal label
  propagation.
\newblock In {\em Proceedings of the 35th International Conference on Machine
  Learning}, pages 5095--5104.

\bibitem[Yang and Leskovec, 2015]{yang2015defining}
Yang, J. and Leskovec, J. (2015).
\newblock Defining and evaluating network communities based on ground-truth.
\newblock {\em Knowledge and Information Systems}, 42(1):181--213.

\bibitem[Yang et~al., 2016]{YangCS16}
Yang, Z., W.~Cohen, W., and Salakhutdinov, R. (2016).
\newblock Revisiting semi-supervised learning with graph embeddings.
\newblock In {\em Proceedings of the 33rd International Conference on Machine
  Learning}, pages 40--48.

\end{thebibliography}

\newpage
\appendix

\section{Closure of SDDM Matrices Under the Schur Complement}
\begin{lemma} If $\mat{M}$ is an SDDM matrix and $T = V \setminus \{x\}$ is a subset of its columns, then $\mat{S} := \mat{SC}(\mat{M},T)$ is also an SDDM matrix.
\end{lemma}
\begin{proof}
	Recall that 
	\[
	\mat{SC}(\mat{M},T) = \mat{M}_{T,T} - \frac{\mat{M}_{T,x}\mat{M}^{\top}_{T,x}}{\mat{D}'_{x,x}},
	\]
	and observe that $\mat{D}'_{x,x} = \mat{M}_{x,x}$. By definition of SDDM matrices, we need to show that $\mat{S}$ is (i) symmetric, (ii) its off-diagonal entries are non-positive, and (iii) for all $i \in [n-1]$ we have $\mat{S}_{ii} \geq - \sum_{j \neq i} \mat{S}_{ij}$. An easy inspection shows that $\mat{S}$ satisfies (i) and (ii). We next show that (iii) holds. 
	
	To this end, by definition of $\mat{S}$, we have that
	\begin{align}
	\label{eq: Soff-diag}
	-\sum_{j \neq i} \mat{S}_{ij} & = \sum_{j \neq i} \left( -\mat{M}_{ij} + \frac{\mat{M}_{ix} \mat{M}_{xj}}{\mat{M}_{xx}} \right) \nonumber \\
	& = -\sum_{j \neq i} \mat{M}_{ij} + \frac{\mat{M}_{ix}}{\mat{M}_{xx}} \left( \sum_{j \neq i} \mat{M}_{xj} \right)  
	\end{align}	
	
	As $\mat{M}$ is an SDDM matrix, the following inequality holds for the $x$-th row of $\mat{M}$
	\[
	-\sum_{j \neq i} \mat{M}_{xj} \leq \mat{M}_{ix},
	\] 
	or equivalently
	\begin{equation}
	\label{eq: helpful} 
	\mat{M}_{ix} \left(\sum_{j \neq i} \mat{M}_{xj} \right) \leq -\mat{M}^2_{ix}.
	\end{equation}
	
	Plugging Eq.~(\ref{eq: helpful}) in Eq.~(\ref{eq: Soff-diag}) and using the fact that $-\sum_{j \neq i} \mat{M}_{ij} \leq \mat{M}_{ii}$, we get that
	\begin{align*}
	-\sum_{j \neq i} \mat{S}_{ij} & \leq \mat{M}_{ii} - \frac{\mat{M}^2_{ix}}{\mat{M}_{xx}} = \mat{S}_{ii},
	\end{align*}
	which completes the proof of the lemma.
\end{proof}

\end{document}